\documentclass{article}
\pdfoutput=1

\usepackage{cite}
\usepackage[pdftex]{graphicx}
\usepackage[cmex10]{amsmath}
\usepackage{amssymb}
\usepackage{amsthm}
\usepackage{algorithmic}
\usepackage[caption=false,font=footnotesize]{subfig}
\usepackage{color}
\usepackage{xfrac}
\usepackage{dsfont}
\usepackage[table,xcdraw]{xcolor}
\usepackage{appendix}

\newtheorem{definition}{Definition}
\newtheorem{thm}{Theorem}
\newtheorem{prop}{Proposition}
\newtheorem{lem}{Lemma}

\newcommand{\m}[1]{\mathcal{#1}}

\newcommand{\e}{\epsilon}

\newcommand{\teq}{\triangleq}

\newcommand{\set}[1]{\left\{#1\right\}}
\newcommand{\sbrk}[1]{\left[#1\right]}
\newcommand{\paren}[1]{\left(#1\right)}

\newcommand{\floor}[1]{\left\lfloor#1\right\rfloor}

\newcommand{\given}[2]{\left.#1\right|#2}

\newcommand{\mP}[1]{\mathbb{P}\left[#1\right]}
\newcommand{\cP}[2]{\mathbb{P}\left[\given{#1}{#2}\right]}

\newcommand{\figref}[1]{Figure \ref{#1}}
\newcommand{\secref}[1]{Section \ref{#1}}
\newcommand{\algref}[1]{Algorithm \ref{#1}}
\newcommand{\lemref}[1]{Lemma \ref{#1}}
\newcommand{\thmref}[1]{Theorem \ref{#1}}

\newcommand{\defref}[1]{Definition \ref{#1}}
\newcommand{\tabref}[1]{Table \ref{#1}}

\newcommand{\algorithmicnote}{\textbf{note:}}
\newcommand{\NOTE}{\item[\algorithmicnote]}

\begin{document}
%
% paper title
% can use linebreaks \\ within to get better formatting as desired
\title{Multi-user lax communications:\\ a multi-armed bandit approach}

% author names and affiliations
% use a multiple column layout for up to three different
% affiliations
\author{Orly Avner\and Shie Mannor}

% make the title area
\maketitle

\begin{abstract}
Inspired by cognitive radio networks, we consider a setting where multiple users share several channels modeled as a multi-user multi-armed bandit (MAB) problem. The characteristics of each channel are unknown and are different for each user.
Each user can choose between the channels, but her success depends on the particular channel chosen as well as on the selections of other users: if two users select the same channel their messages collide and none of them manages to send any data. Our setting is fully distributed, so there is no central control. As in many communication systems, the users cannot set up a direct communication protocol, so information exchange must be limited to a minimum.
We develop an algorithm for learning a stable configuration for the multi-user MAB problem. We further offer both convergence guarantees and experiments inspired by real communication networks, including comparison to state-of-the-art algorithms.
\end{abstract}

\section{Introduction}
The inspiration for this paper comes from the world of distributed multi-user communication networks, such as cognitive radio networks. These networks consist of a set of communication channels with different characteristics, and independent users whose goal is to transmit over these channels as efficiently as possible.

Modern networks, such as cognitive radio networks, must cope with several challenges. First and foremost, the networks' distributed nature prohibits any form of central control. In addition, many users operate on an ``ad hoc'' basis, preventing them from forming inter-user communication. In fact, they probably do not even know how many users share their network.

On top of these issues of multi-user coordination, the channel characteristics may be initially unknown, and differ between users. Thus, learning must be integrated into the solution.

\subsection{Cognitive radio networks}
Cognitive Radio Networks (CRNs), introduced in \cite{Mitola1999}, have attracted considerable attention in recent years. The idea that lies at the heart of CRNs is that advanced sensing mechanisms and increased computation power may enable radio devices to dramatically improve their performance in terms of resource utilization, resilience and more. Networks of such users are usually dynamic and stochastic, giving rise to many interesting problems \cite{Akyildiz2008,Haykin2005}. We focus on developing a sensing and transmission scheme that enables users to learn a stable, orthogonal configuration without communicating directly.

\subsection{Multi-armed bandits}
A well known framework for learning in CRNs is the classical Multi-Armed Bandit (MAB) model. MABs offer a simple, intuitive framework for learning the characteristics of a number of unknown options in an online manner, while balancing exploration and exploitation. A MAB problem consists of a single user repeatedly choosing between arms with different characteristics, that are initially unknown. After every round, the user acquires a reward that depends on the arm she chose. Her goal in most setups is to maximize the expected sum of rewards acquired over time.

As suggested in \cite{Jouini2010}, the channels of a CRN are naturally cast as the arms of a bandit, with different performance measures (bandwidth, ACK signals, bit rate) serving as the reward.

Many papers propose solutions for the stochastic MAB problem (see, e.g., \cite{Auer2002a,Auer2010,Garivier2011}) and its adversarial version (see, e.g., \cite{Auer2002b}), but they all assume a single user is sampling the arms of the bandit.

However, this assumption does not apply in multi-user networks. In the multi-user MAB model, users compete over the arms of \emph{the same} bandit. As a result, they are bound to experience collisions (i.e., multiple users sampling the same arm), unless they employ some form of collision avoidance or coordination mechanism. Collisions in communication networks result in performance degradation, corresponding to reward loss in the MAB model. In order to avoid reward loss, the presence of multiple users must be addressed. We survey several approaches to this issue in \secref{sec:prevWork}.

\subsection{Extension of the CRN-MAB setting}
The novelty introduced in our paper lies in the combination of bandit learning, multiple users, different reward distributions for different users and no direct communication.
The combination of these last two demands - different distributions and no direct communication, poses a real challenge.

As explained in detail in \secref{sec:opt} and in \secref{sec:SM_form},
the only thing we can guarantee in terms of network behavior in this setup is stability. In a dynamic, distributed network, stability is of great value. Once a network has reached a stable configuration, users can focus on utilizing its resources, rather than engaging in coordination or learning efforts; a stable network is more robust and efficient.

Reaching stability is a nontrivial task, since users must learn their channel characteristics while coordinating their actions with the other users, based on very limited observations.

\subsection{Previous work}\label{sec:prevWork}
We now present several approaches to the CRN-MAB problem, coming from different areas and disciplines.

Our problem may be viewed as an assignment problem, i.e., maximum weight matching in a weighted bipartite graph. Users correspond to agents, channels to tasks, and rewards are simply the complementary of the costs of graph edges. Several papers have been published on the distributed assignment problem, but to the best of our knowledge none of them offers a solution for our problem. The well-known Hungarian method \cite{Kuhn1955} requires full knowledge of the graph (i.e., channel characteristics) and assumes the existence of central control. The Bertsekas auction algorithm \cite{Bertsekas1988} frees us from the need for central control, at the cost of direct communication between nodes. The classical Gale-Shapley algorithm \cite{Gale1962} solves the problem of finding a stable marriage configuration, but does not take the need to learn into account. Some papers have actually applied it to CRNs, but not in the learning context \cite{Cohen2013,Leshem2012}. Another work on distributed stable marriage, that makes use of a variant of the Gale-Shapley algorithm, is \cite{Floreen2010}. While it is quite foreign to our problem, the potential function defined in the paper is helpful in our analysis. Another noteworthy work in this context is \cite{Amira2010}. The authors address the challenge of limiting communication between nodes to a minimum, and propose two communication models. Nevertheless, they allow more communication than we would like, and their formulation does not consider learning. Two additional results that deal with distributed stable marriage offer lower bounds and state that \emph{some} form of information exchange is inevitable when solving such problems \cite{Gonczarowski2014,Kipnis2009}.

The papers closest to ours in spirit are those dealing with multi-user MABs. There has been work on the case of reward distributions that do not vary between users, such as \cite{Anandkumar2011} and \cite{Avner2014}. The latter introduces an algorithm that is able to cope with a variable number of users.
Another paper, that addresses different reward distributions for different users, is \cite{Kalathil2014}. Here, the authors employ the Bertsekas auction algorithm. This approach enables users to reach a reward-maximizing solution, at the price of direct, frequent communication between themselves. We further elaborate on the difference between our approach and the approach of \cite{Kalathil2014} in \secref{sec:experiments}.

To this end, we would like to point out that communication between users is undesirable not only because of its price in terms of network resources and time. Once users depend on communication, they are more vulnerable to intentional attacks that may disrupt it, as well as noise bursts that are common in CRNs.

\section{Model and formulation}
We now describe the model, the assumptions accompanying it and our goal.

\subsection{System and users}
We model a communication network with $K$ channels, servicing $N$ independent users. Our work is based on the assumption that $K\geq N$, which is reasonable since without it, implementing a time division based mechanism is necessary. Once such a mechanism is applied, the assumption that $K\geq N$ is valid again.
Time is slotted and users' clocks are synchronized, also a mild assumption for modern communication systems.

The communication network consists of $K$ channels, where only one user can transmit over a certain channel during a single time slot. Each transmission yields a reward, which we assume to be stochastic.

The users are a group of $N$ independent, selfish agents. Their observations are local, consisting only of the history of their actions and rewards. In addition, they do not know the number of users they share a network with. There is no central control managing their use of the network, and they do not have direct communication with each other.

A key characteristic of our model is that the expected reward a channel yields depends not only on the identity of the channel, but also on the identity of the user. Formally, the rewards of the channels are Bernoulli random variables with expected values $\set{\mu_{n,k}}$, where $n\in\set{1,\ldots,N}$ and $k\in\set{1,\ldots,K}$. This property reflects the fact that in real-life users may experience location-based disturbances, manifested in different reward distributions for the same channel.

We model the users' sharing resources through the representation of the communication network by a \emph{single} bandit. This means that two users attempting to access the same channel at the same time, will experience a collision. In our model, the result of a collision is complete loss of communication for that time slot for the colliding users, i.e., zero reward. A user $n$ that accesses a channel $k$ alone during a certain time slot will receive a reward drawn i.i.d. from a Bernoulli distribution with expected value $\mu_{n,k}$. Throughout the paper, we use the term \emph{configuration} to refer to a mapping of users to channels.

\subsection{Limited coordination}\label{ssec:comm}
In an effort to keep our model faithful to real world CRNs, we limit the coordination between users to a minimum.
Thus, users can only transmit in a channel of their choice, or sense the spectrum range and receive binary feedback regarding all channels $\set{1,\ldots,K}$ at time $t$. A ``0'' represents no transmission in channel, while ``1'' stands for the opposite.

\subsection{Reward maximizing solution}\label{sec:opt}
We adopt a system-wide view for characterizing the optimal solution. The optimal configuration must be orthogonal (i.e., no more than one user per channel), in order to avoid collisions and the resulting reward loss. One common approach seeks to maximize the sum of rewards over all users, over time. The assignment of users to channels is chosen accordingly:
$$
  R^* = \max_{\pi\in\m{C}}\sum_{n=1}^N \mu_{n,\pi\paren{n}}
,$$
where $\m{C}$ is the set of all possible permutations of subsets of size $N$ chosen without replacement from the set $\set{1,\ldots,K}$.

However, reaching such a solution requires frequent information exchange. Assume channel $k$ is optimal for two different users $m$ and $n$, but $\mu_{m,k} > \mu_{n,k}$. To maximize the system-wide reward, user $n$ must step down and choose a different channel. The lack of central control requires explicit information exchange regarding the values of $\mu_{m,k}$ and $\mu_{n,k}$, for $m$ and $n$ to decide which of them should step down. Since the reward estimates are updated as time goes by, such preferences must be communicated repeatedly.

Due to limited information exchange, a reward-maximizing solution cannot be guaranteed in our setup. We therefore focus on convergence to a stable, orthogonal configuration.

\subsection{Stable marriage solution}\label{sec:SM_form}
Our goal is to develop policies that will lead users to a stable configuration. We employ the notion of stable marriage to formally define stability:
\begin{definition}\label{def:SMC}
  A Stable Marriage Configuration (SMC) is an assignment of users to channels such that no two users would be willing to swap channels, had they known the true values of the expected rewards.
  Formally, for a pair of users $n$,$m$:
    \begin{align*}
    S_1 &\triangleq \paren{\mu_{n,a_n} < \mu_{n,a_m}}
    && \text{user n would like to swap} \\
    S_2 &\triangleq \paren{\mu_{m,a_m} \leq \mu_{m,a_n}}
    && \text{user m is willing like to swap},
  \end{align*}

  where $a_m$ and $a_n$ are the users' current actions.
  In an SMC,
  \begin{align*}
    S_1 \land S_2 = 0 \quad \forall n,m.
  \end{align*}
\end{definition}

\subsection{Goal}\label{sec:goal}
Given a system with $K$ channels and $N$ users, allowing only limited communication as described in \secref{ssec:comm}, our goal is to reach a configuration that is
orthogonal:  no two users use the same channel, and an SMC, according to \defref{def:SMC}.

\section{Coordination protocol}\label{sec:coord}
Our coordination protocol balances the limitations of \secref{ssec:comm} with the users' need for information exchange by introducing a signalling mechanism between pairs of users. At predefined time slots, a user wishing to occupy a channel may transmit in that channel to express her wish.
In order to ensure that this signal is received by the user currently occupying the channel, we employ a frame-based protocol. We assume users can transmit and sense at the same time, a reasonable requirement in modern communication systems.

The following explanation is best understood by observing \figref{fig:frames}. Our protocol divides time into super frames of length $T_{\text{SF}} = 2 + 2\paren{K-1}$. Each super frame begins with a pair of time slots, $S_1$ and $S_2$, during which a single signalling user, the initiator, is coordinated for the entire super frame. The procedure is described in \algref{alg:CSM-MAB} and in \figref{fig:alg1}. Next come $K-1$ mini-frames of two time slots each, denoted by $S_3$ and $S_4$. Each of these mini-frames corresponds to one channel on the initiator's list of preferred channels. Thus, a single super frame enables one user to go over her entire preference list and signal other users, suggesting they swap channels with her, as explained in \figref{fig:alg2}.
\begin{figure}%[!t]
\centering
\includegraphics[width=0.75\textwidth]{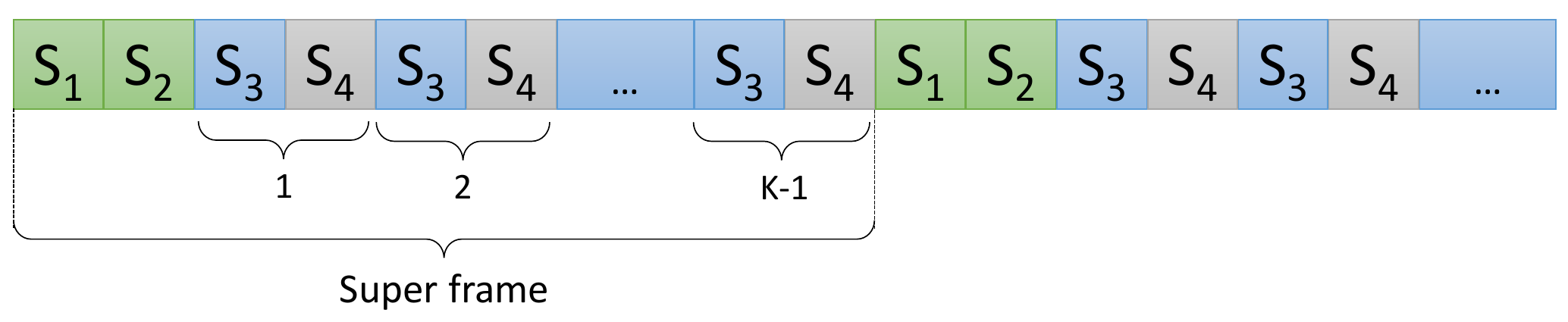}
\caption{Frame structure}
\label{fig:frames}
\end{figure}
\begin{figure}%[!t]
\centering
\includegraphics[width=0.75\textwidth]{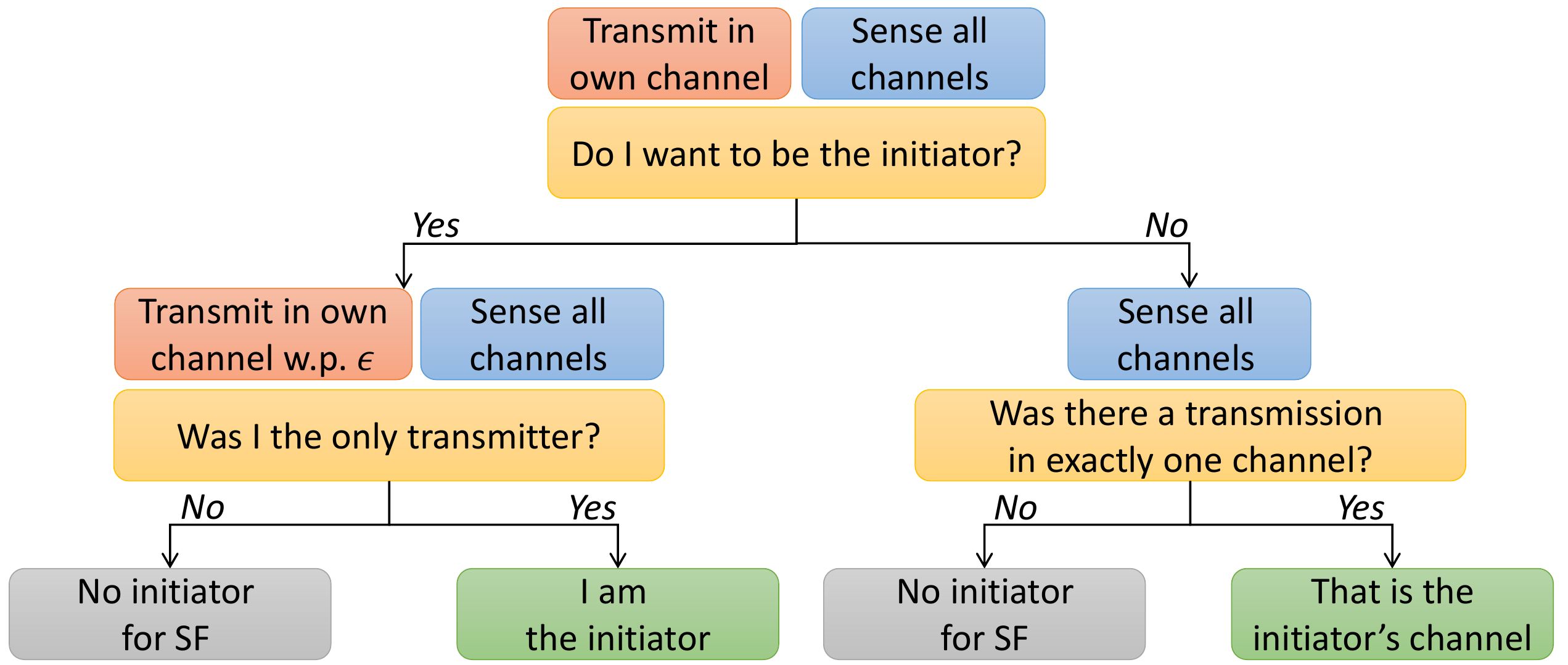}
\caption{Selection of initiator}
\label{fig:alg1}
\end{figure}
The time slots marked $S_4$ allow users not participating in the coordinating process during a certain mini-frame to sample their current channel and proceed with the learning-while-transmitting process. Thus, all but two users (initiator and responder) gather a sample during each mini-frame, resulting in at least $K-2$ samples for each of the users, except for the initiator, over each super frame.
\begin{figure}%[!t]
\centering
\includegraphics[width=0.75\textwidth]{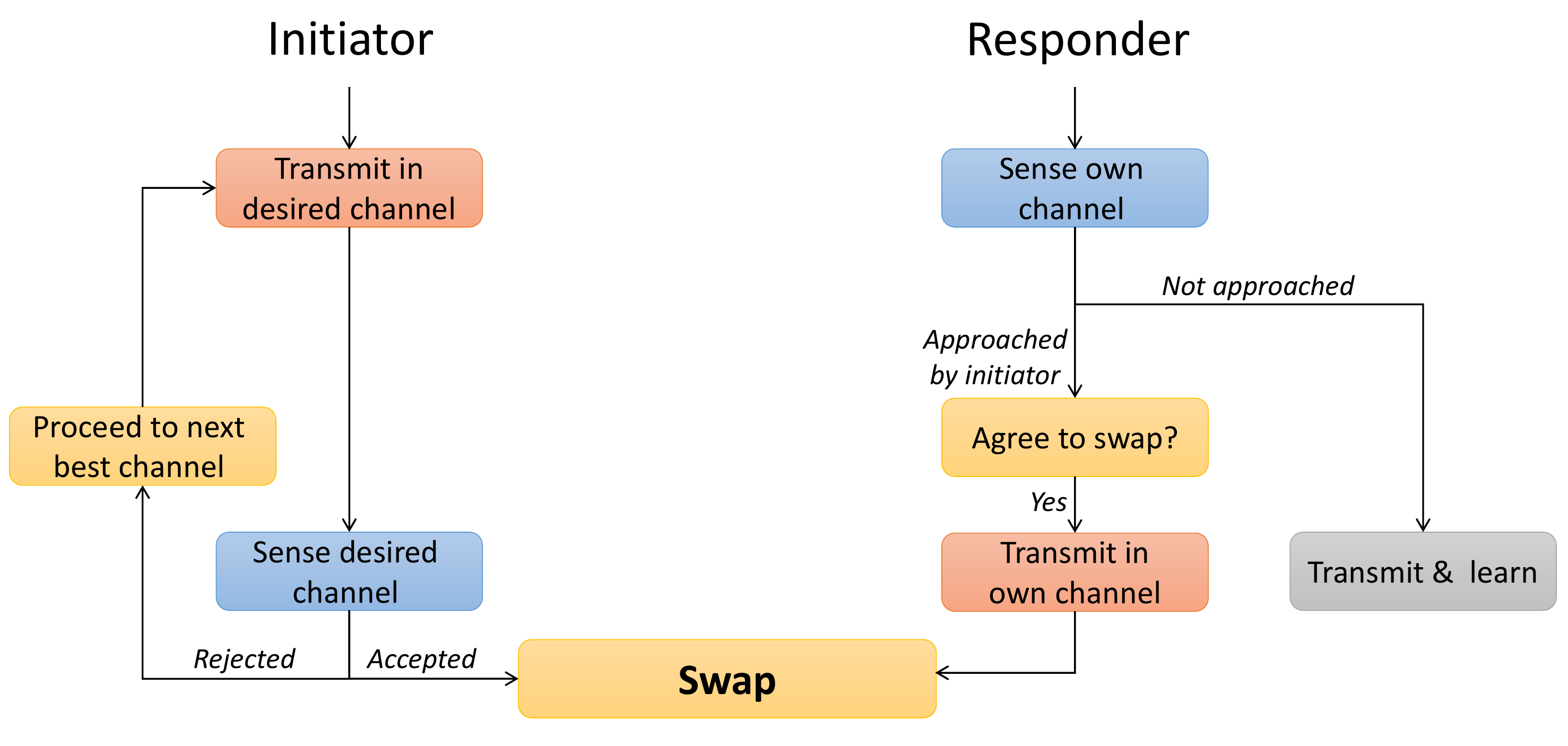}
\caption{Initiator-responder dynamics}
\label{fig:alg2}
\end{figure}

While this may seem like much coordination, the protocol is very simple to implement, and is indeed lightweight when compared to other protocols, as further explained in \secref{sec:analysis} and in \secref{sec:experiments}.

\section{The CSM-MAB algorithm}\label{sec:full_alg}
We now turn to a full description of our algorithm, the Coordinated Stable Marriage Multi-Armed Bandit (CSM-MAB) algorithm. We propose a user-level algorithm for a fully distributed system, whose goal is described in \secref{sec:goal}. When all users in the network apply CSM-MAB, the assignment of users to channels is guaranteed to be orthogonal, and converges to an SMC.

Our algorithm begins with a start up phase, during which users transmit and sense to detect collisions, in order to reach an initial orthogonal configuration (line 1). This phase follows the lines of the CFL algorithm introduced in \cite{Leith2012}, and converges quickly.
Once an initial orthogonal configuration has been reached, users start executing the CSM-MAB algorithm, described in \figref{alg:CSM-MAB}.

\begin{figure}
  \begin{algorithmic}[1]
    \STATE $a_n\paren{0} \gets$ \textbf{apply\_CFL}$\paren{K}$
    \FORALL{frames $t$}
        \IF [Beginning of SF]{$\mod\paren{t,T_{\text{SF}}} == 1$}
            \STATE $list\gets$ \textbf{rank\_channels}$\paren{a_n\paren{t-1},\hat{\mu}_n,s_n}$
            \IF [User seeks to change channel] {$list \neq \mathbf{0}$}
                \STATE $flag_n\gets$ \textbf{rand}$\paren{\text{Bernoulli},\epsilon}$
                \IF {$\paren{flag_n == 1} \land \paren{flag_i == 0 \; \forall i \ne n}$}
                    \STATE $initiator = n$ \COMMENT{User $n$ is initiator for this SF}
                    \STATE $pref = 1$ \COMMENT{Initialize swapping preference to 1}
                \ENDIF
            \ENDIF
        \ELSE
            \IF [$n$ is the initiator, $list$ not exhausted yet]{$ \paren{initiator == n} \land \paren{pref > 0}$}
                \STATE $response \gets$ \textbf{propose\_swap}$\paren{list\paren{pref}}$
                \IF [Responder agreed or channel is available]{$response == 1$}
                    \STATE $a\paren{t} \gets$ \textbf{swap}$\paren{a_n\paren{t},list\paren{pref}}$
                    \STATE $pref \gets 0$
                \ELSE
                    \STATE $pref \gets pref + 1$ \COMMENT{Move to next best channel}
                \ENDIF
            \ENDIF
        \ENDIF
        \STATE $r_n\paren{t}\gets$\textbf{execute\_action}$\paren{a_n\paren{t}}$
        \STATE \textbf{update\_stats}$\paren{r_n\paren{t},\hat{\mu}_{n,a_n\paren{t}},s_{n,a_n\paren{t}}}$
    \ENDFOR
    \NOTE $\hat{\mu}_{n,k}$ is the empirical mean of the reward for user $n$ on arm $k$; $s_{n,k}$ is the number of times she has sampled it.
  \end{algorithmic}
\caption{The CSM-MAB algorithm}
\label{alg:CSM-MAB}
\end{figure}

At the beginning of each super frame, users execute the \textbf{rank\_channels} procedure to individually create a list of channels they prefer over their current action (line 4). Channels are assigned values according to their UCB indices, calculated using the well known formula from \cite{Auer2002a}:
\begin{align}\label{eq:UCBind}
I_{n,k}\paren{t} = \hat{\mu}_{n,k} + \sqrt{\frac{2\ln t}{s_{n,k}}},
\end{align}
where $\hat{\mu}_{n,k}$ is the empirical mean of the reward acquired by user $n$ on channel $k$ up till time $t$ and $s_{n,k}$ is the number of times she sampled arm $k$ up till time $t$.

Next, the users coordinate an initiator according to the scheme in \figref{fig:alg1}. Every user who would like to improve upon her current channel presents herself as the initiator with a probability of $\e = \frac{1}{K}$ (lines 5-11). An agreed initiator for the SF emerges if and only  exactly one user raises her flag (the value of $\e$ is chosen in order to maximize the probability of this occurring). Once a single initiator is agreed upon, all users take note of her current channel, based on their sensing. They will need this knowledge to decide whether to accept her swapping suggestion.

The initiator proceeds to signal other users, based on her ranking of channels (lines 13-21). Signalling is implemented in \textbf{propose\_swap} by transmitting in the initiator's channel of interest. Each responder (i.e., signalled user) checks whether swapping channels with the initiator will improve her situation, based on her own ranking. Once a responder agrees, a swap takes place. No more signalling attempts are made till the end of the SF, and users simply continue sampling their chosen channels.
If the responder refuses, the initiator will approach the next-best channel on her list. She will continue the process until she (a) finds a partner that agrees to swap; or (b) exhausts her list of potential swaps. This part of the algorithm is depicted in \figref{fig:alg2}.

\section{Analysis}\label{sec:analysis}
We will now show that the CSM-MAB meets the goals defined in \secref{sec:goal}.
Our main theoretical result is stated in \thmref{thm:SMC}.

\begin{thm}\label{thm:SMC}
  Consider a system with $K$ channels and $N$ users, with channel rewards characterized by the matrix $\boldsymbol{\mu}$.
  Applying CSM-MAB (\algref{alg:CSM-MAB}) by all users will result in convergence to an orthogonal SMC:
  For all $\delta > 0$ there exists $T\paren{\delta}$ such that for all time slots $t>T$, the probability of the system's being in an SMC is at least $1-\delta$.
\end{thm}

The proof of \thmref{thm:SMC} consists of two aspects: orthogonality and stability. The first part is easy to verify.

\begin{prop}
  The actions of users applying CSM-MAB are orthogonal (i.e., there is at most one user sampling each channel) for all $t>t_0$ with probability of at least $1-\delta_0$.
\end{prop}

\begin{proof}
  Based on Theorem 1 of \cite{Leith2012}, the initial configuration reached after running the CFL algorithm is orthogonal with probability 1.
  The authors provide an upper bound on the distribution of stopping times, $\tau$:
  \begin{align*}
    \mP{\tau > k} = \alpha e^{-\gamma k},
  \end{align*}
  where $\alpha$ and $\gamma$ are some positive constants. The expected stopping time is therefore upper bounded by $\frac{\alpha e^{-\gamma}}{1-e^{-\gamma}}$.
  Thus, setting $t_0 \teq \frac{2\alpha e^{-\gamma}}{1-e^{-\gamma}}$, the probability of not having reached an orthogonal configuration by time $t_0$ is at most $\delta_0 \teq e^{-2\frac{\alpha e^{-\gamma}}{1-e^{-\gamma}}}$.
  Once the system reaches an orthogonal configuration, a user does not switch to an occupied channel without having coordinated the switch, as defined in \algref{alg:CSM-MAB}.
\end{proof}

\subsection{Stability and potential}\label{sec:potential}
Showing that our system converges to a stable solution is more involved. We begin by defining a potential function for the problem. For any user $n\in\set{1,\ldots,N}$, the potential at time $t$ is defined as follows:
\begin{align}\label{eq:user_potential}
  \phi_n\paren{t} \triangleq \sum_{k = 1}^K \mathds{1}\set{\mu_{n,k} > \mu_{n,a_n\paren{t-1}}},
\end{align}
where $a_n\paren{t-1}$ is the action taken by user $n$ in the previous time step.
In words, the potential is the number of channels user $n$ would prefer over her current choice, had she known their true reward distributions.
The system-wide potential is the sum of potentials over all users:
\begin{align}\label{eq:system_potential}
  \Phi\paren{t} \triangleq \sum_{n=1}^N \phi_n\paren{t}
\end{align}
An illustration of the potential appears in Tables 1 and 2.

\begin{table}[ht]
\centering
\caption{Table of users' channel rankings (first row represents best channel, last row represents worst). Cells highlighted in yellow and underline represent user's current choice.}
\label{tab:prefs}
\begin{tabular}{l|c|c|c|}
\cline{2-4}
                                 & \multicolumn{1}{l|}{\textbf{$U_1$}} & \multicolumn{1}{l|}{\textbf{$U_2$}} & \multicolumn{1}{l|}{\textbf{$U_3$}} \\ \hline
\multicolumn{1}{|l|}{\textbf{1}} & 1                                   & 2                                   & \cellcolor[HTML]{FFFE65}\underline{4}           \\ \hline
\multicolumn{1}{|l|}{\textbf{2}} & 2                                   & \cellcolor[HTML]{FFFE65}\underline{1}           & 1                                   \\ \hline
\multicolumn{1}{|l|}{\textbf{3}} & 4                                   & 3                                   & 2                                   \\ \hline
\multicolumn{1}{|l|}{\textbf{4}} & \cellcolor[HTML]{FFFE65}\underline{3}           & 4                                   & 3                                   \\ \hline
\end{tabular}
\end{table}

\begin{table}[ht]
\centering
\caption{User potentials corresponding to the configuration in \tabref{tab:prefs}.}
\label{tab:pot}
\begin{tabular}{|c|c|c|}
\hline
$\phi_1$ & $\phi_2$ & $\phi_3$ \\ \hline
3        & 1        & 0        \\ \hline
\end{tabular}
\end{table}

In terms of potential, a configuration is an SMC if no two users can swap channels and decrease their potential by doing so. We note that a stable configuration does not necessarily correspond to zero system-wide potential, since not all users might be able to achieve zero potential simultaneously, depending on network parameters. Also, a system may have several stable configurations, each characterized by a different potential.
Nevertheless, observing a system's potential does provide an indication regarding stability: once a system reaches a stable configuration, its potential will no longer change.

We prove convergence to an SMC by using the potential function, considering three aspects:
\begin{enumerate}
    \item The maximal potential of a system with $K$ channels and $N$ users is finite and equal to $N\paren{K-1}$.
    \item The potential $\Phi\paren{t}$ is monotonously non-increasing with high probability.
    \item Until an SMC is reached, changes in potential are bound to happen within finite time.
\end{enumerate}
We formalize and prove these statements in the sequel.

Since users' decisions are guided by UCB indices, while stability is examined with respect to true reward distributions, users do not always update their choice of channels in a way that matches the ground truth. Thus, the system potential may occasionally increase, due to users' exploration or inaccurate statistics. In our proof we show that despite this, users ultimately converge to a stable configuration.

\subsection{Proof of \thmref{thm:SMC}}
We begin by ensuring the monotonicity of the potential.
\begin{lem}\label{lem:monoPot}
For all times $t$ for which $t > \frac{16K}{\Delta_{\min}^2}\ln t$, if a change in potential occurs, it is a decrease, with probability of at least $1-2t^{-4}$.
\end{lem}
$\Delta_{\min}$ is a distribution dependent constant.
In the appendix we derive an upper bound on the minimal time for which the condition above holds:
\begin{align}\label{eq:t_min}
  t_{\min} \leq \frac{M-1-\sqrt{\paren{M-1}^2 -4M}}{2},
\end{align}
where $M \triangleq \frac{16K}{\Delta_{\min}^2}$. This bound will enable us to use $t_{\min}$ in the proof.

Next, we introduce a lemma that concerns the ability of a single user to reach the position of the initiator.
\begin{lem}\label{lem:initiator}
If $\phi_n\paren{t} > 0$ for some user $n$, then her probability of becoming the next initiator is at least $\e\paren{1-\e}^{N-1}$.
\end{lem}

Using \lemref{lem:initiator}, we show another result:
\begin{lem}\label{lem:finite_time_till_switch}
If the system is not in an SMC at some time $t$, then a change in the potential will occur within no more than $t^\prime\paren{\delta_1}$ time slots with probability of at least $1-\delta_1$.
\end{lem}
The exact dependency of $t^\prime$ on $\delta_1$ appears in the appendix, as do the proofs of all lemmas.

The probability of the system's reaching an SMC within
$\tau\triangleq t^\prime N\paren{K-1}$ time slots after time $t_{\min}$ is at least
\begin{align*}
  P_{\text{SMC}}\teq \sbrk{\paren{1-\delta_1}\paren{1-2t_{\min}}^{-4}}^{N\paren{K-1}}.
\end{align*}
We model the convergence to an SMC using a Markov chain. Let $S_t$ denote the state of the system at time $t$:
\begin{align*}
  S_t =
  \begin{cases}
    0 & \text{if in SMC},\\
    1 & \text{else}.
  \end{cases}
\end{align*}
The following holds for the chain's transition probability:
\begin{align*}
  \cP{S_{t+\tau} = 1}{S_{t_{\min}} = 0} \geq P_{\text{SMC}},
\end{align*}
and also
\begin{align*}
  \cP{S_T = 0}{S_{t_{\min}} = 0} \leq \paren{1-P_{\text{SMC}}}^{\floor{\frac{T-t_{\min}}{\tau}}},\\
  \forall T > t_{\min} + \tau.
\end{align*}
Defining $\delta \teq \paren{1-P_{\text{SMC}}}^{\floor{\frac{T-t_{\min}}{\tau}}}$ completes the proof, and inverting yields
\begin{align*}
  T = t_{\min} + \tau\frac{\ln \delta}{\ln\paren{1-P_{\text{SMC}}}}.
\end{align*}

Our next result quantifies the time devoted to signalling.
\begin{prop}
  In every super-frame $\paren{K-1}\paren{N-2}$ learning samples are gathered by all users combined. During this period $4K$ signalling and sensing actions are performed by all users combined, so the signalling to learning ratio is
  \begin{align*}
    L\triangleq \frac{4K}{\paren{K-1}\paren{N-2}}.
  \end{align*}
\end{prop}
Clearly, the effort the users put into coordination is most effective when the number of users is close to the number of channels. This is a result of the frames' length being dictated by the number of channels rather than the number of users, in order for the user-level algorithm to be independent of the number of users.

\section{Experiments}\label{sec:experiments}
To demonstrate the merits of our algorithm, we implement a simulation of a distributed multi-user communication network.
The users in our network are synchronized, and time is slotted.

In this network, users cannot communicate with each other directly. However, they can sense the entire frequency range (i.e., listen to all channels). They may also transmit over a channel of their choice, updating this choice each time slot.

A user $n$ transmitting over a channel $k$ receives a binary reward, drawn i.i.d. from a Bernoulli distribution with parameter $\mu_{n,k}$. This can be viewed as a form of the classic binary symmetric channel. As far as the different values of the reward parameters go, we ran experiments in two different modes:
\begin{enumerate}
  \item random: the $\mu_{n,k}$'s are drawn uniformly and independently from the interval $\sbrk{0,1}$.
  \item real-world: users are divided into clusters, and each cluster has a preferred group of channels. This represents a scenario in which users sharing a cluster are geographically close, and experience an interference in part of the frequency range. In real-world wireless communication systems, an agent that does not belong to the network but is transmitting in its vicinity will often cause a similar phenomenon.
\end{enumerate}

We present results obtained in an experiment with $K=12$ channels and $N=10$ users. The users are divided into two clusters. Users 1-5 belong to one cluster, and experience an interference in the frequency range of channels 7-12. Users 6-10, on the other hand, experience similar performance over the entire frequency range.
Experiments last $T=120000$ time slots, and results are averaged over 50 repetitions.

We begin by examining the cumulative number of policy changes per user over time, plotted in \figref{fig:policyChange} and in \figref{fig:policyChangeEmp}.
Since our goal is stability, we would like the number of policy changes to be small, and indeed the rate of changes decreases significantly over time. Another observation, demonstrated by the two figures, is that different users have different patterns, depending on the realization but more importantly on the difficulty of their problem: users that have small differences between channels will need more samples in order to tell them apart, and will therefore experience more policy changes.

\begin{figure}%[!t]
\centering
\includegraphics[width=0.75\textwidth]{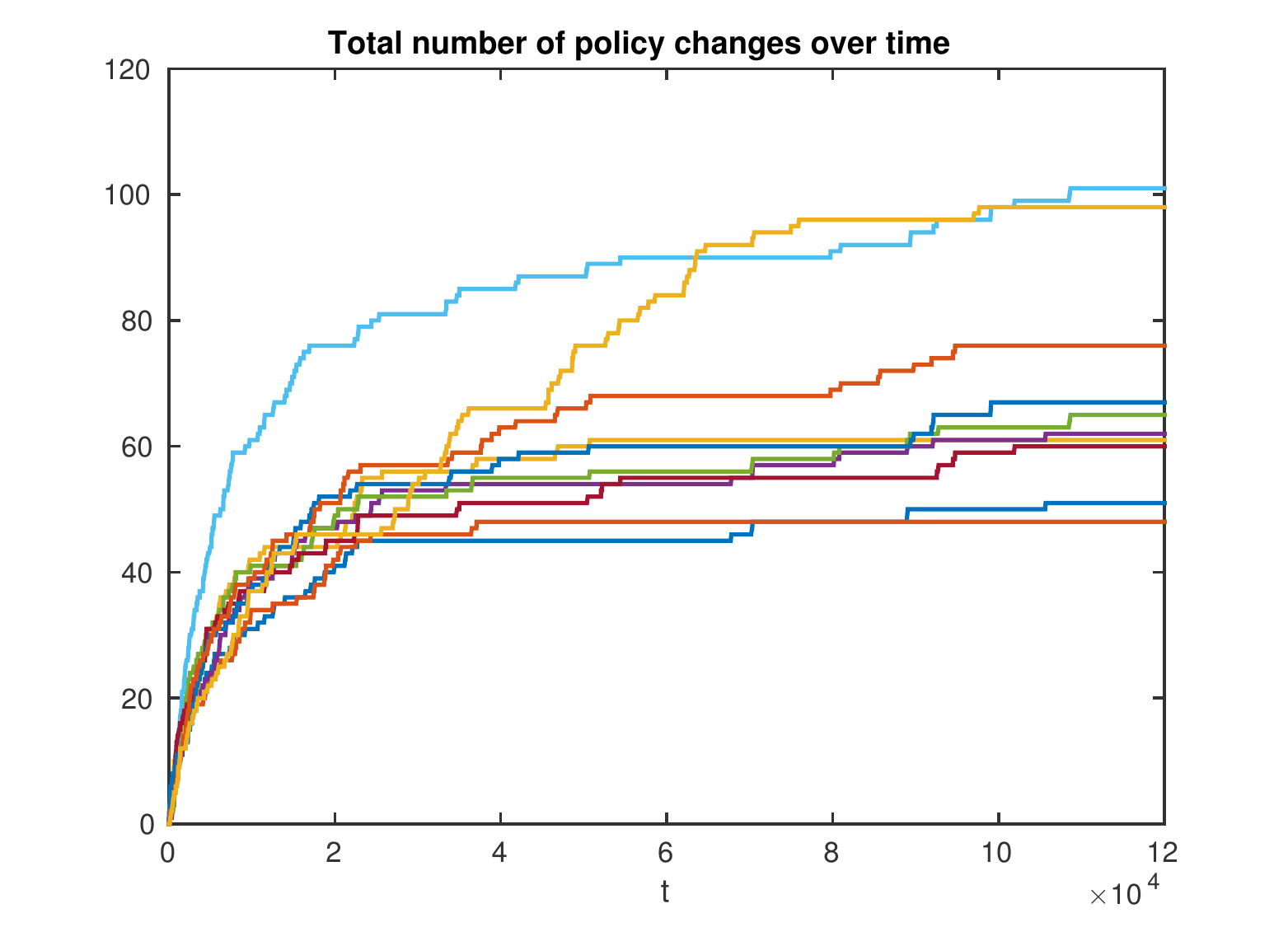}
\caption{Changes in users' choice of channels, single realization}
\label{fig:policyChange}
\end{figure}

\begin{figure}%[!t]
\centering
\includegraphics[width=0.75\textwidth]{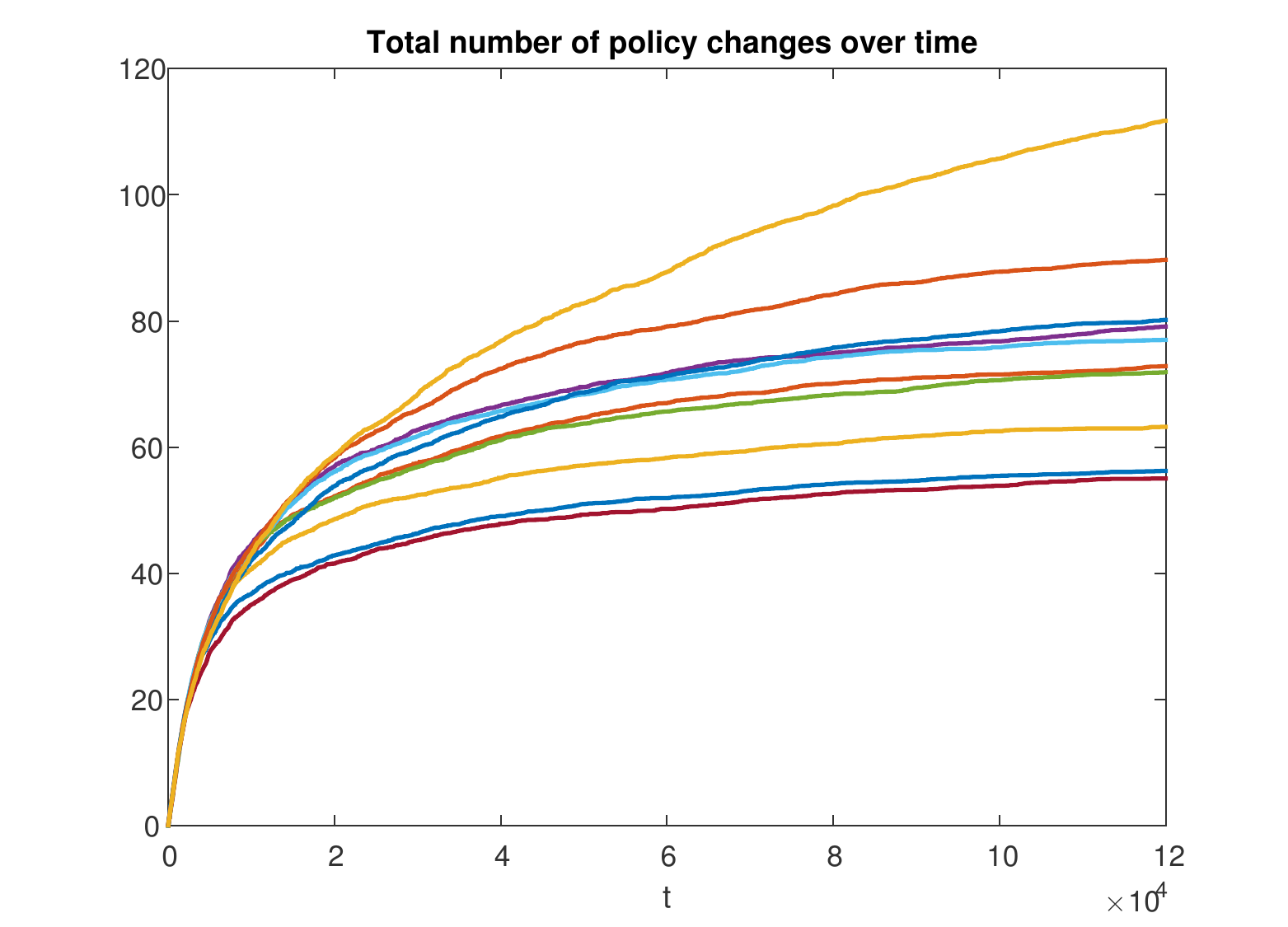}
\caption{Changes in users' choice of channels, empirical average}
\label{fig:policyChangeEmp}
\end{figure}

Our next result examines the convergence to different SMCs over several repetitions of one setup. In this case, the set of SMCs consists of 305 configurations. Naturally, the size of this set depends on the number of users, $N$, the number of channels, $K$, and also on the specific realization of the $\mu_{n,k}$'s.
\figref{fig:imageSM} shows that the periods of time users spend in unstable configurations decrease as the experiment advances, and users move between different SMCs, depending on the realization.

\begin{figure}
\centering
\includegraphics[width=0.75\textwidth]{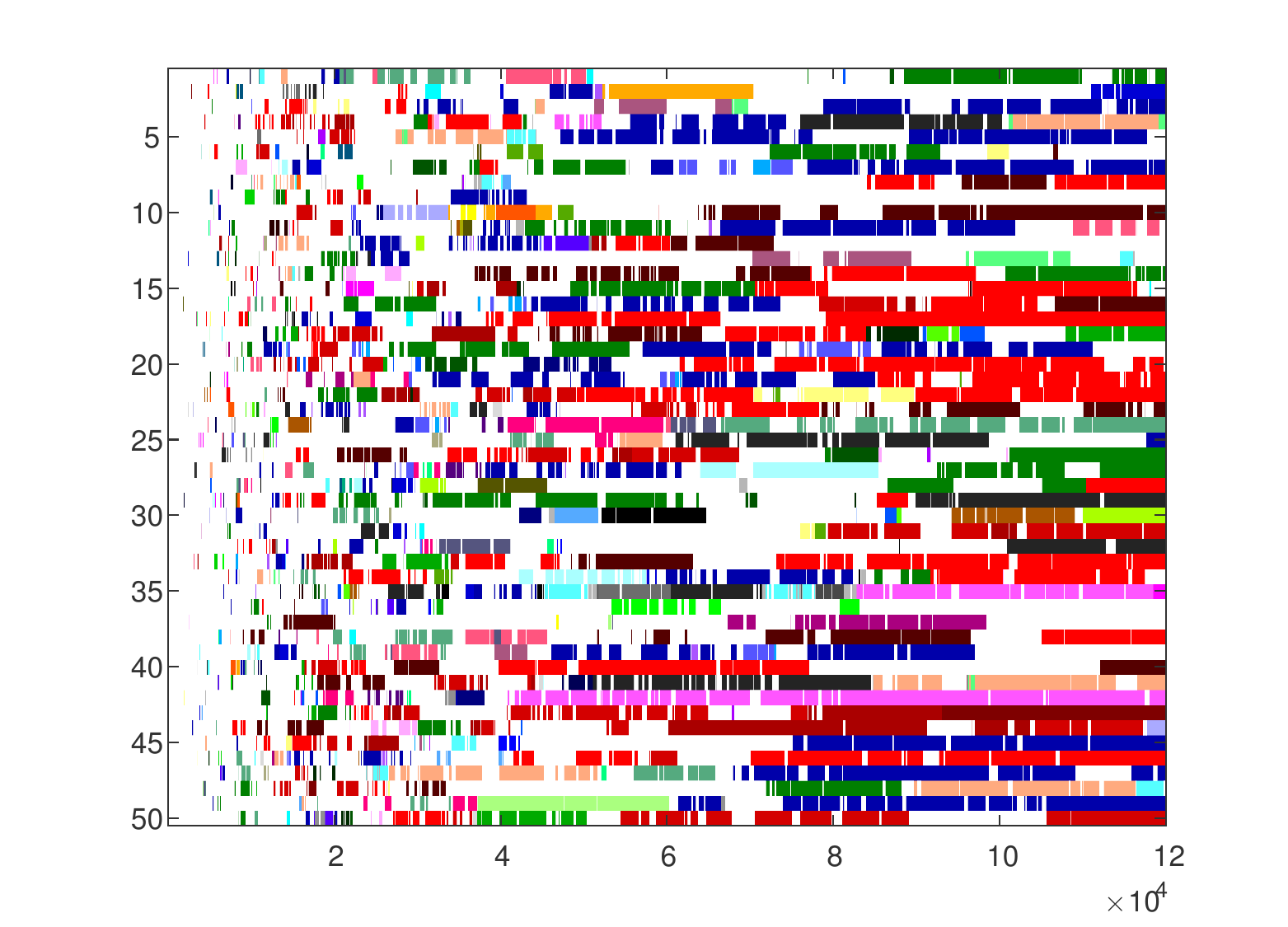}
\caption{Convergence to SMC for different realizations: horizontal axis shows time, vertical axis shows numbering of realizations. White pixels represent unstable configurations, other colors correspond to different SMCs. As time goes by, longer stretches of time are spent in SMCs.}
\label{fig:imageSM}
\end{figure}

To complement our proof, we provide a visualization of the system potential over time, averaged over several repetitions, in \figref{fig:sysPotential}. As shown in the proof, the potential decays on average. The shaded area around the plot represents the variance over iterations, which also decays over time. As explained in \secref{sec:potential}, the potential does not necessarily decay to zero, but rather to a constant value that represents the potential of the SMC. %\todo{plot}

\begin{figure}
\centering
\includegraphics[width=0.75\textwidth]{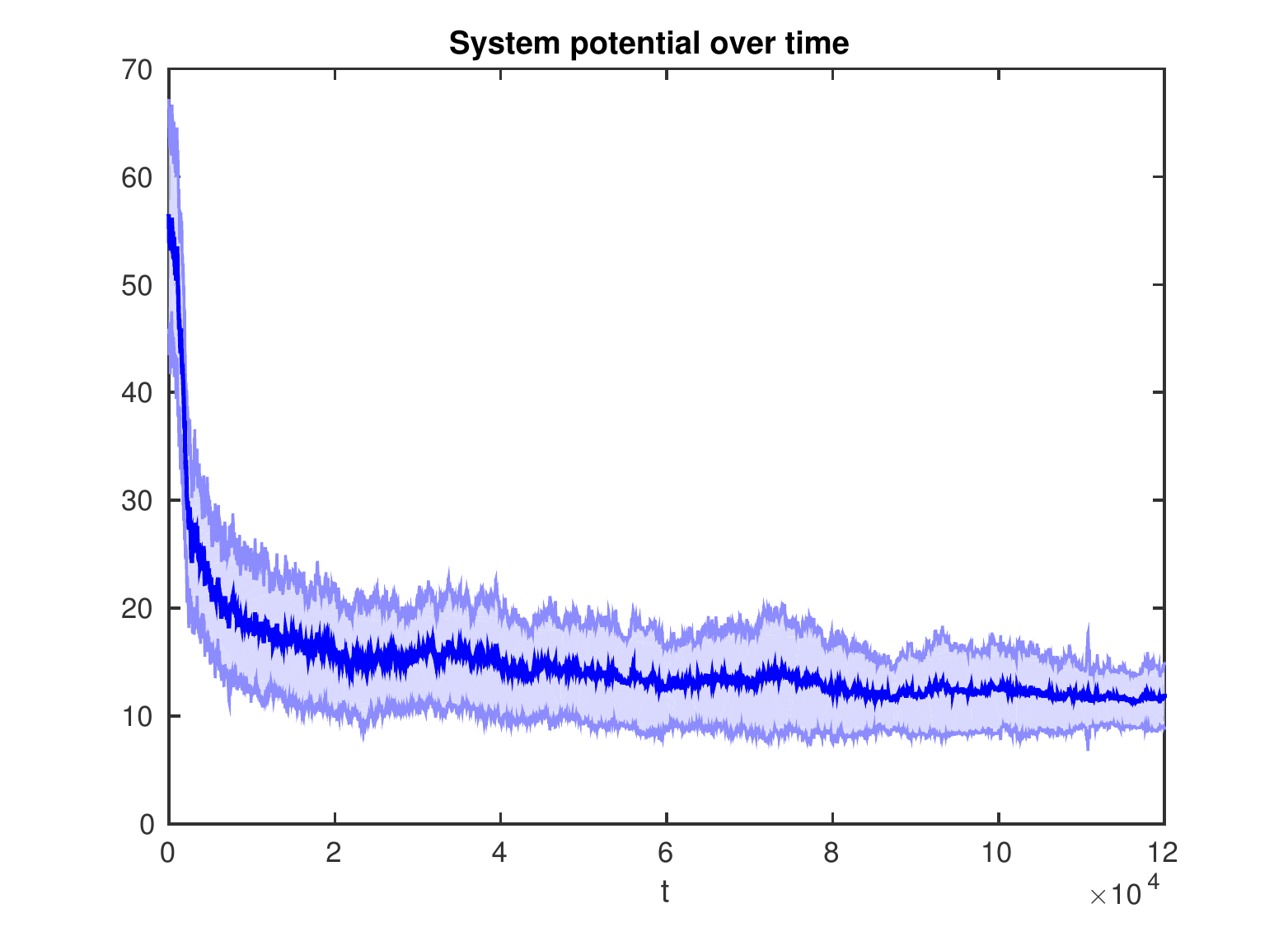}
\caption{Decay of system potential over time, averaged over 50 repetitions. The shaded area represents variance.}
\label{fig:sysPotential}
\end{figure}

Our last result examines the reward acquired by users employing the CSM-MAB algorithm. While our theoretical guarantees focus on stability, the algorithm incorporates reward maximization implicitly by using UCB indices to rank channels. However, as explained in \secref{sec:opt}, reaching a reward-optimal configuration cannot be guaranteed with the limited form of communication we allow. In \figref{fig:reward} we compare the cumulative system-wide reward of two algorithms: our CSM-MAB and the dUCB4 algorithm, introduced in \cite{Kalathil2014}. As explained in \secref{sec:prevWork}, dUCB4 incorporates an auction algorithm in order to achieve an orthogonal reward maximizing configuration.

The price of reward maximization is, clearly, communication, which our scheme attempts to bring to a minimum. In order to implement the auction algorithm required by dUCB4, users must have distinct id's and knowledge of the number of users. This rather technical requirement hinders the ability of the algorithm to deal with a variable number of users. Our algorithm naturally extends to a scenario in which users arrive and leave at random times, that is quite likely in the context of CRNs. In addition, auction algorithms inherently rely on the good will of users, and are therefore more vulnerable to malicious agents (e.g., agents that report false high bids for attractive channels).

The results in \figref{fig:reward} demonstrate the tradeoff between communication and reward maximization: the time dUCB4 invests in auctioning is quite dominant. The two variants of the algorithm differ in the accuracy of the auctioning algorithm. The ``dUCB4'' variant (dotted red) uses 32 bits to encode variables, while the ``dUCB4Long'' variant (dashed magenta) uses 64 bits. Because of auctioning, it takes the algorithm a long time to turn its focus to reward maximization. In the high-accuracy case, the users exhaust all their time auctioning. In the low-accuracy case, they only begin acquiring rewards towards the end of the experiment. In real-world networks, with constantly changing conditions, such a long start-up phase is difficult to overlook. For the sake of example, let us examine an average 802.11n WLAN network, with a nominal frame size of 2000 bits and typical bit rate of 25 megabits per second. The $4\cdot10^5$ time slots it takes dUCB4 to start acquiring rewards are translated into a period of $\frac{4\cdot10^5\cdot2000}{25\cdot10^6}=32\text{sec}$. This start-up phase doubles to over one minute when 64 bit accuracy is used for the auction algorithm. Of course, lighter schemes than the 802.11 can be used, but these numbers clearly demonstrate the potentially crippling overhead brought on by communication.

\begin{figure}
\centering
\includegraphics[width=0.75\textwidth]{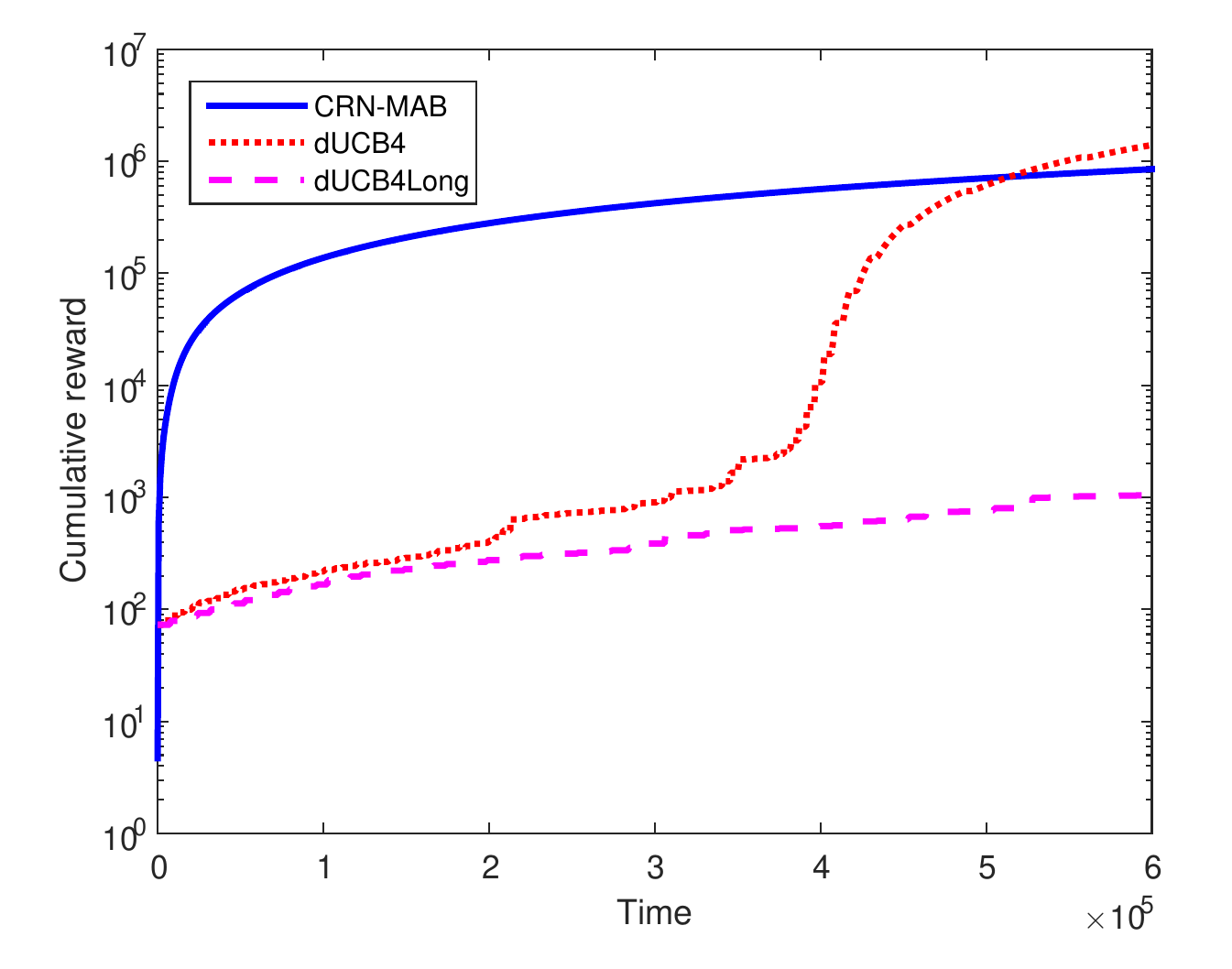}
\caption{Cumulative system-wide reward over time, for different algorithms}
\label{fig:reward}
\end{figure}

We note that when $N$ is strictly less than $K$, our algorithm often reaches the reward optimal configuration, or a configuration very similar in reward values. Therefore, the variance of the cumulative reward is very small. Our intuitive explanation is that when $N<K$ users have a certain degree of freedom, increasing their chances of landing in the optimal configuration.

Despite reaching a configuration that is very close to optimal in the presented simulations, our algorithm acquires reward at a slower rate than dUCB4, due to the constant ratio of coordination and exploitation. Decreasing the amount of time devoted to coordination may considerably increase the reward, at the cost of impairing the algorithm's ability to handle a variable number of users. We plan to address this issue in detail in the future.

\section{Discussion}\label{sec:discussion}
We present an extension of the multi-user MAB problem, for the case of different reward distributions between the users, with limited information exchange.
Using a specialized signalling method, our algorithm enables multiple users to learn network characteristics and converge to an orthogonal configuration that is also a stable marriage.
We provide a theoretical analysis of our algorithm's performance, based on the notion of system potential.
Finally, we present the results of an experimental setup and examine different aspects of our approach's performance, including a comparison to the dUCB4 algorithm of \cite{Kalathil2014}. As explained in \secref{sec:experiments} in further detail, the main difference between the algorithms is the way they strike a balance between minimizing communication and maximizing the reward. We argue that our algorithm is better suited for real world problems.

In the future we intend to extend our work to a dynamic scenario, both in terms of channel characteristics and number of users. The latter should be straightforward due to the minimal inter-dependency of users, while the former will require some adjustment of the learning algorithm.
Another interesting variant, applicable to networks with a fixed number of users, alters the ratio between coordination and exploitation as time goes by, to enable better use of network resources.

\bibliographystyle{IEEEtran}
\bibliography{main}

\appendix\label{sec:appendix}
\appendixpage

\section{Proof of \lemref{lem:monoPot}}
We would like to show that for all values of $t$ for which $t > \alpha\ln t$, the probability that the potential decreases every time it changes is at least $1-4t^{-4}$, where $\alpha = \frac{32K}{\Delta_{\min}^2}$.

Given that a change in potential occurs at time $t$, it is guaranteed to result in a potential decrease if it benefits both users. This will happen if both users' indices, that guide their decisions, are accurate w.r.t the true distribution.

Since we condition on a change in potential,
\begin{align*}
  \mP{\Phi_{\text{Dec}}}  = 1 - \mP{\Phi_{\text{Inc}}}
\end{align*}
Let us upper bound $\mP{\Phi_{\text{Inc}}}$ . For a user $n$ switching from arm $j$ to arm $i$ at time $t$, $\mu_{n,i} < \mu_{n,j}$,
\begin{align*}
  \mP{\Phi_{\text{Inc}}} = \mP{I_{n,i}\paren{t} \geq I_{n,j}\paren{t} \cap
  \mu_{n,i} < \mu_{n,j}},
\end{align*}
where $I_{n,i}\paren{t}$ is user $n$'s UCB index of arm $i$ at time $t$, defined in \eqref{eq:UCBind}.
Following the proof of Theorem 1 of \cite{Auer2002a},
\begin{align*}
  \mP{\Phi_{\text{Inc}}}
  &= \mP{\hat{\mu}_{n,i}\paren{t} + c_{t,s_{n,i}}
  \geq \hat{\mu}_{n,j}\paren{t} + c_{t,s_{n,j}} \cap \mu_{n,i} < \mu_{n,j}} \\
  &\leq 2t^{-4},
\end{align*}
provided that
\begin{align}\label{eq:condUCB}
  s_{n,i} \geq \frac{8\ln t}{\Delta_{i,j}^2\paren{n}},
\end{align}
where $s_{n,i}$ is the number of times user $n$ sampled arm $i$ up till time $t$ and $\Delta_{i,j}\paren{n} \triangleq \mu_{n,i} - \mu_{n,j}$.
If \eqref{eq:condUCB} does not hold, then the UCB index ``misleads'' user $n$, causing her to mistakenly favor arm $i$, despite its lower expected reward. Switching from arm $j$ to arm $i$ will result in an increase in potential. However, once she acquires another sample of arm $i$, its index will decrease. In the meantime, the index of arm $j$ will increase due to the passing time, and the indices will ultimately reflect the correct preference, resulting in a potential decrease.

The extreme value for \eqref{eq:condUCB}, i.e., the largest number of required samples, corresponds to the minimal value of $\Delta_{i,j}\paren{n}$.
Let us define:
\begin{align*}
  \Delta_{n} &\triangleq
  \min_{\substack{i,j\in\set{1,\ldots,K} \\ i\neq j}}\sbrk{\mu_{n,i} - \mu_{n,j}}\\
  \Delta_{\min} &\triangleq \min_{n\in\set{1,\ldots,N}}\Delta_{n}
\end{align*}
Thus, when all arms have been sampled at least
\begin{align}\label{eq:s_min}
  s_{\min} \triangleq \frac{8\ln t}{\Delta_{\min}^2}
\end{align}
times, the probability of an increase in potential is very small.

In order to allow for the coordination protocol, users do not gather informative samples in every time slot. Instead, they gather at least $K-2$ samples in each super frame, whose length is $T_{\text{SF}} = 2 + 2\paren{K-1}=2K$.

Therefore, taking into account the fact that the sampling condition in \eqref{eq:s_min} must apply for all arms, the condition on $t$ is
\begin{align}\label{eq:tCond}
  t > K\frac{T_{\text{SF}}}{K-2}s_{\min} = \frac{16K^2}{\paren{K-2}\Delta_{\min}^2}\ln t > \frac{16K}{\Delta_{\min}^2}\ln t.
\end{align}
For all times $t$ for which \eqref{eq:tCond} holds, if a change in potential occurs, it is a decrease, with probability of at least $1-2t^{-4}$.

When we apply this lemma we will use a quantity $t_{\min}$, an upper bound on the minimal $t$ for which \eqref{eq:tCond} holds.
Introducing a well-known lower bound on the logarithmic function:
\begin{align*}
  \ln x \geq \frac{x-1}{x+1} \quad \forall x >1.
\end{align*}
We use this lower bound together with \eqref{eq:tCond}:
\begin{align*}
  t_{\min} = \frac{16K}{\Delta_{\min}^2}\ln t_{\min}
  \geq \frac{16K}{\Delta_{\min}^2} \frac{t_{\min} -1}{t_{\min} +1}.
\end{align*}
Denoting $M \triangleq \frac{16K}{\Delta_{\min}^2}$, we continue:
\begin{align*}
  t_{\min} &\geq M \frac{t_{\min} -1}{t_{\min} +1} \\
  t_{\min}^2 + \paren{1-M}t_{\min} + M &\geq 0.
\end{align*}
Our conclusion is that $t_{\min} \leq \frac{M-1-\sqrt{\paren{M-1}^2 -4M}}{2}$. Since this expression is finite, we may now use it in our proof.

\section{Proof of \lemref{lem:initiator}}
The probability of a specific user becoming the initiator when there are $\ell$ interested users is
\begin{align*}
  P_s\paren{\e,\ell}&\triangleq\cP{\text{specific initiator}}{\ell\text{ interested}}\\
   &= \e\paren{1-\e}^{\paren{\ell-1}} \; \forall \ell\in\set{1,\ldots,N}.
\end{align*}
The probability is minimized when all $N$ users would like to become the initiator, yielding the bound $\e\paren{1-\e}^{N-1}$.

\section{Proof of \lemref{lem:finite_time_till_switch}}
  If the system has not reached an SMC, then according to \defref{def:SMC}, the conditions $S_1$, $S_2$ hold for at least one pair of users $n,m$.

  According to the definition of the CSM-MAB algorithm, if $S_1$ holds, then user $n$ will add the channel user $m$ is sampling to her list of preferred channels with a probability of at least $1-\delta$. Following arguments similar to those presented in the proof of \lemref{lem:monoPot}, $\delta < 2t^{-4}$. If $S_2$ holds, user $m$ will accept user $n$'s swap proposal, assuming her statistics are accurate. This, once again, happens with a probability of at least $1-\delta$. Once users $n$ and $m$ swap channels, the potential will change.

  In the worst case (i.e., largest $t'$), user $m$'s channel will be the last channel on user $n$'s list, and all users higher on the list will decline user $n$'s swap proposals. If user $n$ approaches a different user (whose channel is ranked higher than $m$'s), and that user agrees to swap, the potential will also change.

  What is left to prove is that the time it shall take user $n$ to receive the privilege of being initiator is finite.
  Once $n$ is appointed the initiator, it will take no more than $K-1$ mini-frames, i.e., $2\paren{K-1}$ time slots, until she approaches user $m$ and a swap takes place.

  There are two different cases - if $n,m$ are the the only unstable pair, then they will be the only ones interested in becoming the initiators. Furthermore, if only one of them is dissatisfied, then there will only be one user interested in initiating. In the notation of \lemref{lem:initiator}, this corresponds to $\ell = 2$ or $\ell = 1$, respectively. The probability of exactly one of them becoming the initiator is $P_{1,2}=\min\set{\e,2\e\paren{1-\e}}$.

  If there are additional unstable pairs, there will be more nominees for initiating. However, not all super frames necessarily result in a decrease in potential - if the initiator only targets channels occupied by ``satisfied'' users, all her attempts will be rejected.
  Therefore, we need to address the worst case scenario, in which all $N$ users attempt to initiate, but only one of them is in a position that will actually result in a swap. Based on \lemref{lem:initiator}, the probability of that user emerging as the single initiator is at least $\e\paren{1-\e}^{N-1}$, for a single super frame. This probability is smaller than $P_{1,2}$ for all $\e,N$, and is therefore the lower bound for the probability of a single initiator with actual capacity for a decrease in potential.

  The number of SFs in a time interval of length $t^\prime$ is $C = \floor{\frac{t^\prime}{T_{\text{SF}}}}$. The probability that a single initiator with actual capacity for a decrease in potential \emph{does not} emerge in a certain SF is less than $1-\e\paren{1-\e}^{N-1}$, and the probability that a single initiator does not emerge in the interval is less than $P_C \triangleq\paren{1-\e\paren{1-\e}^{N-1}}^C$. As $t'\to\infty$, so does $C$, and $P_C$ decays to zero.

  Binding the two aspects of this lemma together, we have that the probability of a single initiator with actual capacity for coordinating a switch emerging in an interval of length $t^\prime$ is at least $1-P_C$. The probability of a swap between users whose actions do not correspond to a stable configuration is at least $\paren{1-2t^{-4}}^2$. The combined result: if the system is not in an SMC at time $t$, then a change in the potential will occur within no more than $t^\prime$ time slots with probability of at least $\paren{1-P_C}\paren{1-2t^{-4}}^2$, where $P_C \triangleq\paren{1-\e\paren{1-\e}^{N-1}}^{\floor{\frac{t^\prime}{T_{\text{SF}}}}}$.

  Let us re-write the result for the sake of clarity: if the system is not in an SMC at time $t$, then a change in the potential will occur within no more than $t^\prime\paren{\delta_1}$ time slots with probability of at least $1-\delta_1$.
  Developing the previous expression for the probability of a change in potential:
  \begin{align*}
    \paren{1-P_C}\paren{1-2t^{-4}}^2
    &= \paren{1-P_C}\paren{1-4t^{-4}+4t^{-8}}\\
    &\geq \paren{1-P_C}\paren{1-4t^{-4}} \\
    &= 1-P_C - 4t^{-4} + 4P_C t^{-4} \\
    & \geq 1 - P_C - 4t_{\min}^{-4}.
  \end{align*}
  From now on, we denote $\delta_1 = P_C + 4t_{\min}^{-4}$.
  Using this, we can derive an expression for $t^\prime\paren{\delta_1}$:
  \begin{align*}
    P_C &= \delta_1 - 4t_{\min}^{-4} \\
    \paren{1-\e\paren{1-\e}^{N-1}}^{\floor{\frac{t^\prime}{T_{\text{SF}}}}}
    &= \delta_1 - 4t_{\min}^{-4} \\
    \frac{t^\prime}{T_{\text{SF}}}\ln \paren{1-\e\paren{1-\e}^{N-1}}
    &= \ln\paren{\delta_1 - 4t_{\min}^{-4}} \\
    t^\prime & = T_{\text{SF}}\frac{\ln\paren{\delta_1 - 4t_{\min}^{-4}}}{\ln \paren{1-\e\paren{1-\e}^{N-1}}}.
  \end{align*}

\end{document}